\pdfoutput=1

\documentclass[11pt]{article}

\usepackage{emnlp2021}

\usepackage{times}
\usepackage{latexsym}
\usepackage{amsmath}
\usepackage{amssymb}
\usepackage{amsthm}
\usepackage{inconsolata}
\usepackage[capitalize]{cleveref}
\usepackage{graphicx}
\usepackage{comment}
\usepackage{enumitem}
\usepackage{bbold}
\usepackage{listings}
\usepackage{float}

\DeclareMathOperator*{\argmin}{\arg\min}
\newcommand{\expected}{\mathbb{E}}
\newtheorem{lemma}{Lemma}
\newtheorem{proposition}{Proposition}

\usepackage[T1]{fontenc}

\usepackage[utf8]{inputenc}

\usepackage{microtype}

\newcommand\lm{p_\mathrm{LM}}
\newcommand\lmp{p_{\mathrm{LM}'}}
\newcommand\hyp{\tilde{p}}
\newcommand{\Xg}{X_\mathrm{G}}
\newcommand{\Xl}{X_\mathrm{L}}

\newcommand{\defeq}{\triangleq}
\newcommand\pglobal{\tilde{p}_\mathrm{G}}
\newcommand\plocal{\tilde{p}_\mathrm{L}}
\newcommand\pinterpadd{\tilde{p}_+^\lambda} \newcommand\pinterpmul{\tilde{p}_\times^{\lambda_1, \lambda_2}}

\title{
   How Do Neural Sequence Models Generalize? \\
   Local and Global Context Cues for Out-of-Distribution Prediction
}

\author{Anthony Bau \and Jacob Andreas \\
  Computer Science and Artificial Intelligence Laboratory \\
  Massachusetts Institute of Technology \\
  \{abau, jda\}@mit.edu}

\begin{document}
\maketitle

\begin{abstract}
    After a neural sequence model encounters an unexpected token, can its behavior be predicted?
    We show that RNN and transformer language models exhibit structured, consistent generalization in out-of-distribution contexts.
    We begin by introducing two idealized models of %
    generalization in next-word prediction: a \emph{local context model} in which generalization is consistent with the last word observed, and a \emph{global context model} in which generalization is consistent with the global structure of the input. In experiments in English, Finnish, Mandarin, and random regular languages, we demonstrate that neural language models \emph{interpolate} between these two forms of generalization: their predictions are well-approximated by a log-linear combination of local and global predictive distributions. We then show that, in some languages, \emph{noise} mediates the two forms of generalization: noise applied to input tokens encourages global generalization, while noise in history representations encourages local generalization.
    Finally, we offer a preliminary theoretical explanation of these results by proving that the observed interpolation behavior is expected in log-linear models with a particular feature correlation structure.
    These results help explain the effectiveness of two popular regularization schemes and show that aspects of sequence model generalization can be understood and controlled.
\end{abstract}

\section{Introduction}
\label{sec:intro}

\begin{figure}[t!]
    \centering
    \includegraphics[width=0.95\columnwidth,clip,trim=0 0 0 0.1in]{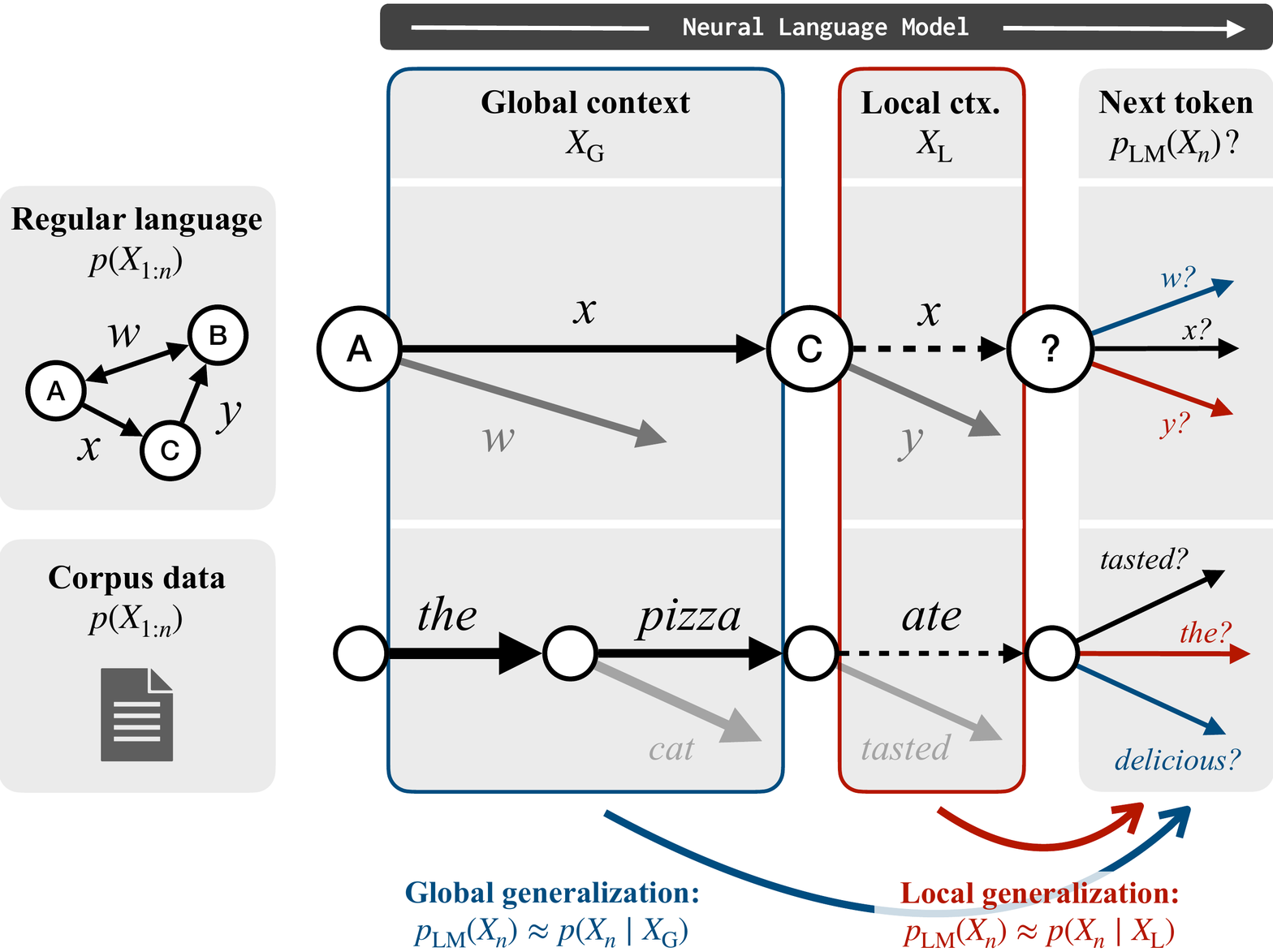}
    \includegraphics[width=\columnwidth,clip,trim=0 2.9in 0 0]{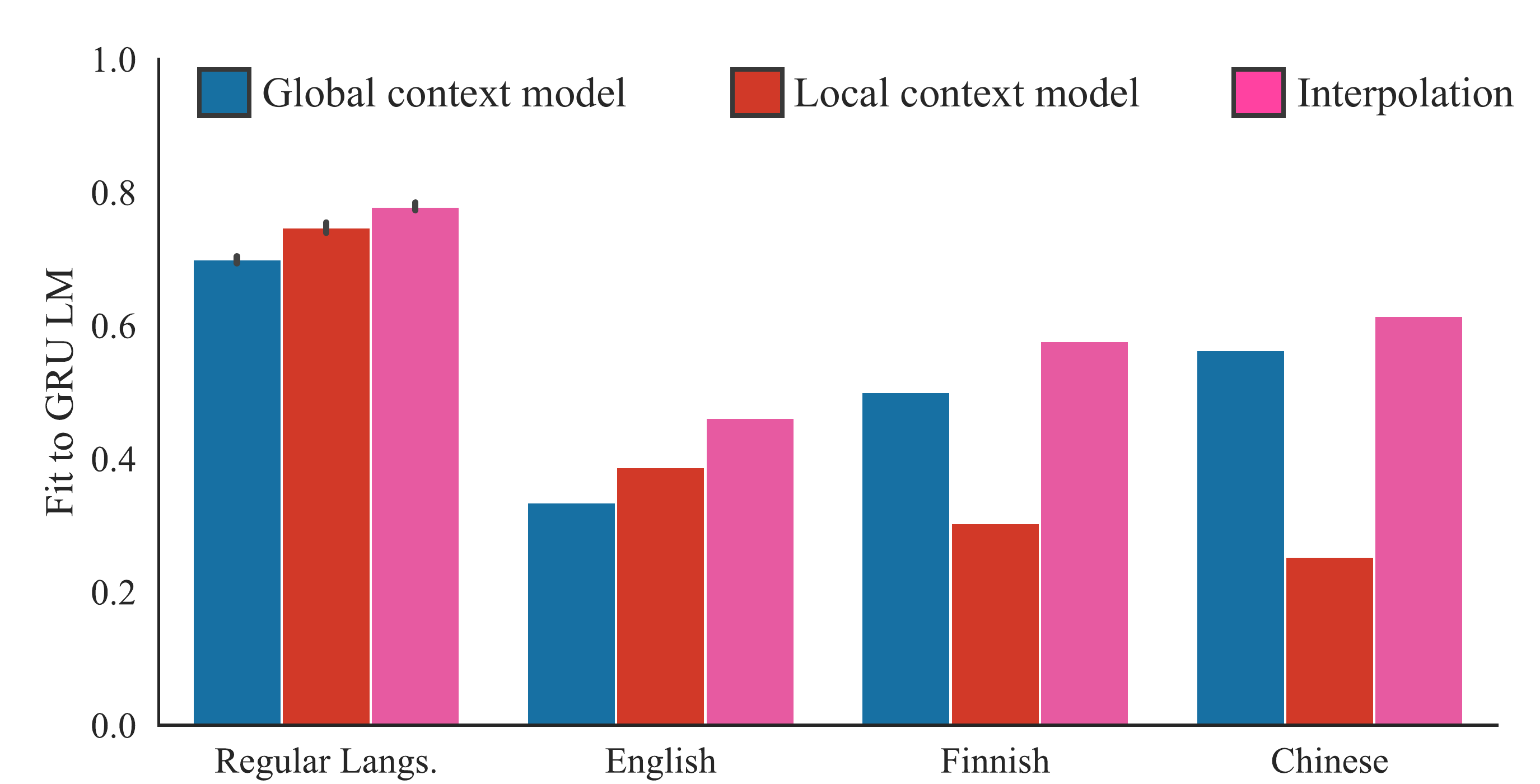}
    \caption{We develop formal models of the predictions of neural language models in \emph{surprising contexts} in which local information (e.g.\ the most recent token) and global information (e.g.\ the rest of the sentence) conflict (top). In these out-of-distribution contexts, predictors trained on both synthetic and natural languages favor either local or global information, but are best approximated by an \emph{interpolation} of a local-only and global-only predictor (bottom).}
    \label{fig:teaser}
\end{figure}

Neural language models (LMs) play a key role in language processing systems for tasks as diverse as machine translation, dialogue, and automated speech recognition \citep{baziotis2020language,sordoni2015neural,mikolov2010recurrent}.
These LMs, which model distributions over words in context via recurrent, convolutional, or attentional neural networks,
have been found to consistently outperform finite-state approaches to language modeling based on hidden Markov models \cite{kuhn1994ergodic} or $n$-gram statistics \cite{miller1950verbal}.
But improved predictive power comes at the cost of increased model complexity and a loss of transparency. While it is possible to characterize (and even control) how finite-state models will behave in previously unseen contexts, generalization in neural LMs is not nearly as well understood.

Consider the following sentence prefixes:
\begin{itemize}[noitemsep]
    \item[(a)] \textit{The pandemic won't end children can\ldots}
    \item[(b)] \textit{Let him easter\dots}
    \item[(c)] \textit{After we ate the pizza, the pizza ate\dots}
\end{itemize}
Each of these prefixes should be assigned a low probability under any reasonable statistical model of English: (a) is missing a word, (b) has a noun used in place of a verb, and (c) a features a selectional restriction violation.\footnote{In fact, all three are examples of naturally occurring text: (a) was originally published (as a typo) on the New York Times homepage \cite{nyttweet}, (b) appears in \citet{hopkins1918} and (c) in by \citet{bench2013}.}
When exposed to these surprising contexts, what word will language models predict next? For finite-state models of language, the answer is clear: $n$-gram models back off to the shortest context in which statistics can be reliably estimated (e.g.\ just the final word; \citealt{katz1987estimation}), and hidden Markov models explicitly integrate the possibility of an unexpected part-of-speech transition and an unexpected word choice \cite{freitag1999information}.
But in neural models, model behavior in-distribution provides little insight into behavior in novel contexts like the ones shown in (a--c).

Characterizing neural LMs' behavior on inputs like these is important for many reasons---including evaluating their robustness, characterizing their effectiveness as models of human language processing, and identifying inductive biases relevant to deployment in new tasks.
This paper offers three steps toward such a characterization:

\begin{enumerate}
    \item 
        We present an empirical description of neural LM behavior in out-of-distribution contexts like the ones shown in (a--c). We introduce two idealized models of prediction in these contexts:  a \emph{local context model} in which generalization is consistent with the last word observed (ignoring global sentence structure), and a \emph{global context model}, in which generalization is consistent with the global structure of the input (ignoring unexpected words).
        In experiments on English, Finnish, Mandarin, and a collection of random regular languages,
        we show that neural LM behavior is reasonably well approximated by either the local or global context model, and even better predicted by an \emph{interpolation} of the two: neural LMs reconcile conflicting information from local and global context by modeling their contributions independently and combining their predictions post-hoc (\cref{fig:teaser}).
        
    \item
        We further show that, in regular languages, \emph{noise} introduced at training time modulates the relative strength of local and global context in this interpolation: input noise (in the form of random word substitution) encourages global generalization, while history noise (dropout applied to recurrent states or self-attention layers) encourages local generalization. These effects are small, but point toward a potential role for noise-based regularization schemes in controlling out-of-distribution behavior.
        
    \item
        Finally, we offer a preliminary mathematical explanation of the observed results by demonstrating that this interpolation behavior arises in any regularized log-linear model with separate local and global context features that are individually predictive of future tokens.
\end{enumerate}

Despite the complexity of current neural LMs, these results show that aspects of their out-of-distribution generalization can be characterized, controlled, and understood theoretically.

\section{Background}
\label{sec:background}

\paragraph{Generalization in count-based LMs}

Before the widespread use of neural approaches in NLP, statistical approaches to language modeling were typically \emph{defined} by explicit independence assumptions governing their generalization in contexts never observed in the training data. For example, $n$-gram models \citep{miller1950verbal,shannon1951prediction} ignore global sentence structure in favor of a local context of at most $n$ words. By contrast, latent-variable language models based on finite-state machines \citep{kuhn1994ergodic} (or more expressive automata; \citealt{chelba1998exploiting}, \citealt{pauls2012large}) explicitly incorporate information from the long-range context by conditioning next-word prediction on abstract global states constrained by global sentence structure.
In models of both kinds, behavior in contexts unlike any seen at training time is be explicitly specified via \emph{backoff} and \emph{smoothing} schemes aimed at providing robust estimates of the frequency of rare events \cite{good1953population,katz1987estimation,kneser1995improved}. Like past work on backoff and smoothing, our work in this paper attempts to provide a general mechanism for both prediction and control in more complex, black-box neural LMs.

\paragraph{Generalization in feature-based and neural LMs}

Such mechanisms are necessary because,
with the advent of feature-rich approaches to language modeling---including log-linear models \citep{rosenfeld1996maximum} and neural network models \citep{bengio2003neural,mikolov2010recurrent,vaswani2017attention}---the kinds of structured, engineered generalization available in finite-state models of language have largely been lost. Current models clearly generalize to new linguistic contexts (including those with semantic content very different from anything seen at training time; \citealt{radford2019language}). But the precise nature and limits of that generalization---especially its robustness to unusual syntax and its ability to incorporate information about global sentence structure---remain a topic of ongoing study.

Current work largely focuses on controlled, linguistically motivated tests of generalization: measuring models' ability to capture long-range agreement, movement, and licensing phenomena on diagnostic datasets \citep[][]{gauthier2020syntaxgym}.
For example, \citet{linzen-etal-2016-assessing} show that while RNNs are capable of storing the information necessary to enforce subject--verb agreement, the language modeling training objective does not encourage it; \citet{mccoy2020does} demonstrate that RNN models for a question formation task favor linear generalizations over hierarchical ones (roughly, lexical generalizations over syntactic ones) on out-of-distribution inputs. Rather than focusing on a specific language or class of linguistic phenomenon, our work in this paper aims to provide a general-purpose framework for reasoning about generalization in neural sequence models across contexts and languages.

\paragraph{Generalization beyond NLP}

The generalizations investigated in paper involve instances of
covariate shift---a change in the distribution $p(x)$ for a conditional model $p(y \mid x)$---which has been extensively investigated in more general machine learning settings \citep[e.g.][]{storkey2009training}.
Outside of NLP,
there have been several attempts to describe more abstract inductive biases native to RNNs and transformers, including work 
focused on compositionality \citep{liska2018memorize,lake2018generalization,weber2018fine} and even more generic algorithmic priors
\citep{lan2021minimum, kharitonov2020they}. 
Here we focus on the architectures and context shifts relevant to language processing tasks.
We validate our models of generalization using real models trained on natural data and explain them in terms of measurable properties of these data distributions.

\section{Models of Generalization}
\label{sec:models}

Consider the example contexts shown in (a--c). Each is an extremely unlikely sentence prefix, featuring text that is globally inconsistent with English syntax or semantic constraints. In such contexts, is it possible to predict \emph{a priori} what a neural LM trained on language data will do next?

We can formalize the situation depicted in these examples as follows: Let $p(X_{1:n}) = p(X_1, X_2, \ldots, X_n)$ be a distribution over sentences with tokens $X_i$, and let
\begin{equation}
    \lm(X_{1:n}) = \prod_{i=1}^n \lm(X_i \mid X_{1:i-1})
\end{equation}
be a learned approximation to this distribution produced by an autoregressive model of the conditional distribution $\lm(X_n \mid X_{1:n-1})$.
We will consider each \textbf{context} $X_{1:n-1}$ to comprise a \textbf{global context} $\Xg = X_{1:n-2}$ (all but the last word) and a \textbf{local context} $\Xl = X_{n-1}$ (the last word in the context). Then, given some thresholds $\epsilon$ and $\tau$, we will call a context $(\Xg, \Xl)$ \textbf{surprising} if for some thresholds $\epsilon$ and $\tau$,
\begin{align}
\label{eq:eps}
    \lm(\Xl | \Xg) < \epsilon \\
    \intertext{(the juxtaposition of $\Xg$ and $\Xl$ is low-probability), while:}
    \lm(\Xl) > \tau \\
    \intertext{and each for each $i$}
    \lm({\Xg}_{,i} \mid {\Xg}_{,1:i}) > \tau
    \label{eq:tau}
\end{align}
($\Xg$ and $\Xl$ are high-probability marginally). In example (c), $\Xg = $ \emph{the pizza}, $\Xl =$ \emph{ate}. Given a language model $\lm$, we wish to understand whether $\lm(X \mid \Xg, \Xl)$ has systematic or predictable structure in surprising contexts---can it be explained in terms of statistics of the underlying distribution $p$ or the behavior of $\lm$ in unsurprising contexts?

In the remainder of this section, we describe a set of candidate hypotheses about what this next-token distribution might look like, and in \cref{sec:experiments} evaluate the extent to which these hypotheses accurately predict the true behavior of $\lm$.

\subsection{Local and global models of generalization}
\label{sec:models-base}

We focus on two idealized models of the generalization that might be exhibited by neural language models.

\paragraph{Local context model}

In this model, we hypothesize that predictors reconcile the conflicting information from $\Xg$ and $\Xl$ by ignoring the global component of the context, and making the next-token distribution \emph{locally} consistent with the last token seen, regardless of global sentence structure. We denote this model of generalization $\plocal$:
\begin{equation}
    \plocal(X_n \mid \Xg, \Xl) \defeq p(X_n \mid \Xl) ~ .
\end{equation}
$\plocal$ implements a form of backoff common in $n$-gram language models: faced with a long context in which the data distribution is unknown, models discard long-range information and use higher-quality estimates from a shorter context.
We previously defined $\Xl = X_{n-1}$, so experiments with $\plocal$ will predict that neural LMs behave like bigram models; this could be naturally generalized to local contexts consisting of more than a single word.

$\plocal$ can also be viewed as the hypothesis that NLMs implement a particular kind of lossy-context model \citep{futrell-levy-2017-noisy,futrell2020lossy}, who note that ``local contextual information plays a privileged role in [human] language comprehension''; as we will see, this appears to be the case for some neural models as well. Sequence models with backoff may also be given a hierarchical Bayesian interpretation \citep{teh2006bayesian}.

\paragraph{Global context model}

As an alternative, we consider the possibility that predictors rely \emph{exclusively} on the global component of the context, ignoring the unexpected final token:
\begin{align}
    \pglobal(X_n \mid \Xg, \Xl) &\defeq p(X_n \mid \Xg) \nonumber \\
    &= \sum_v p(X_n \mid X_{n-1} = v, \Xg) \nonumber \\
    & ~~~~ \times p(X_{n-1} = v \mid \Xg) ~ .  \hspace{-2pt}
\end{align}
In the language of count-based models, this amounts to the hypothesis that NLMs generalize as \emph{skip-gram} models \citep{goodman2001bit,guthrie2006closer}, performing a kind of reverse backoff to context prior to the most recent word. In the global context model, it is the most recent word, and not the rest of the context, that as treated as a possible source of noise to be marginalized out rather than conditioned on.

\subsection{Interpolated models}
\label{sec:models-interp}

Even when combined in surprising ways, both the local and global context are likely to carry useful information about the identity of the next word. Indeed, models and features implementing both kinds of context representation have been found useful in past work on language modeling \citep{goodman2001bit}.
It is thus natural to consider the possibility that neural LMs \emph{interpolate} between the local context and global context models, combining evidence from $p(X_n \mid \Xl)$ and $p(X_n \mid \Xg)$ when there is no evidence for the specific context $p(X_n \mid \Xl, \Xg)$.

We consider two ways in which this evidence might be combined:

\paragraph{Linear interpolation} In this model,
\begin{align}
    \pinterpadd \defeq \lambda \cdot \plocal + (1 - \lambda) \cdot \pglobal ~ .
\end{align}
Here we predict generalization according to a direct weighted combination of $\plocal$ and $\pglobal$, with the relative importance of the two hypotheses controlled by a parameter $\lambda \in [0, 1]$. Informally, this hypothesis assigns non-negligible probability to next tokens that are consistent with either base hypothesis. Similar interpolation schemes were proposed for $n$-gram modeling by \citet{jelinek1980interpolated}.

\paragraph{Log-linear interpolation} In this model,
\begin{align}
\label{eq:llinterp}
    \pinterpmul &\stackrel{\triangle}{\propto} \pglobal^{\lambda_1} \cdot \plocal^{\lambda_2} ~ .
\end{align}
That is
\begin{align}
    \log \pinterpmul &= \lambda_1 \log \pglobal + \lambda_2 \log \plocal - \log Z
\end{align}
for $\lambda_1$ and $\lambda_2 \in [0, 1]$ and some contextual normalizing constant $Z$ that depends on $X_{1:n-1}$. Here, probabilities from the two base hypotheses are added in log-space then renormalized; informally, this has the effect of assigning non-negligible probability to next tokens that are consistent with \emph{both} base hypotheses. A similar approach was proposed for count-based language modeling by \citet{klakow1998log}.

\section{Experiments}
\label{sec:experiments}

Which of these models (if any) best describes the empirical behavior of neural LMs trained on real datasets? In this section, we present two sets of evaluations. The first aims to characterize how well $\plocal$, $\pglobal$, and combinations of the two predict the out-of-distribution behavior of RNN \citep{elman1990finding} and transformer \cite{vaswani2017attention} language models with standard training. The second explores whether these training procedures can be modified to control the relative strength of local and global generalization.

Both sets of experiments investigate the behavior of RNN and transformer LMs on a diverse set of datasets: first, a collection of \emph{random regular languages} in which the true data distribution $p(X)$ can be precisely modeled; second, a collection of \emph{natural language} datasets from three languages (Mandarin Chinese, English, and Finnish) which vary in the flexibility of their word order and the complexity of their morphology. We begin with a more detailed discussion of models and datasets in \cref{sec:experiments-prelim}; then describe generalization experiments in \cref{sec:experiments-base} and control experiments in \cref{sec:experiments-interp}.

\subsection{Preliminaries}
\label{sec:experiments-prelim}

\paragraph{Data: Formal languages}

The first collection of evaluation datasets consists of a family of \emph{random regular languages}. We begin by generating three deterministic finite automata, each with 8 states and a vocabulary of 128 symbols. Using the algorithm in \cref{appendix:proofs}, we randomly add edges to the DFA to satisfy the following constraints: (1) every state is connected to approximately 4 other states, and (2) each symbol appears on approximately 4 edges. States are marked as accepting with probability $\frac{1}{2}$.

Experiments on these carefully controlled synthetic languages are appealing for a number of reasons. First, because we have access to the true generative process underlying the training data, we can construct arbitrarily large training sets and surprising evaluation contexts $(\Xg, \Xl)$ that are guaranteed to have \emph{zero} probability under the training distribution, ensuring that our experiments cleanly isolate out-of-distribution behavior. Second, the specific construction given above means that ``evidence'' for the local and global models of generalization is balanced: no tokens induce especially high uncertainty over the distribution of states that can follow, and no states induce especially high uncertainty over the set of tokens they can emit, meaning that a preference for local or global generalization must arise from the model rather than the underlying data distribution.

In experiments on these datasets, we generate \emph{training} examples via a random walk through the DFA, choosing an out edge (or, if available, termination) uniformly at random from those available in the current state. We generate \emph{surprising} \emph{test} examples by again sampling a random walk, then appending a symbol that \emph{cannot} be produced along any out-edge from that random walk's final state. We compute $\plocal$ and $\pglobal$ using the ground-truth distribution from each DFA.

In regular languages, the local context model thus hypothesizes that \emph{lexical} information governs out-of-distribution prediction, predicting that LM outputs are determined by the set of states attached to an edge labeled with the surprising symbol. Conversely, the global context model hypothesizes that \emph{structural} information governs out-of-distribution prediction: LM outputs are determined by the set of states reachable from the last state visited before the surprising symbol.

RNN experiments use gated recurrent units \citep{cho2014learning} with a word embeddings of size 128 and a single hidden layer of size 256. Transformer experiments use a hidden size of 256, 4 self-attention layers, and ReLU nonlinearities. Both models are trained with the Adam optimizer and a learning rate of 3e$-$4 on 128,000 examples.

\paragraph{Data: Natural languages}

The second collection of evaluation datasets use natural language data. We conduct experiments on English, Finnish, and Mandarin Chinese. These languages exhibit 
varying degrees of morphological complexity and freedom of word order, with Finnish at one extreme (morphologically complex and freely ordered) and Mandarin at the other.

English data comes from the WMT News Crawl corpus \cite{barrault-etal-2019-findings}. We used a 20,000-sentence subset of sentences from articles from 2007, tokenized using the SentencePiece byte-pair encoding \cite{sentencepiece} with a vocabulary size of $2^{14}$. We used a 2,000-sentence held-out set for validation.
Finnish data comes from the Turku Dependency Treebank \cite{Haverinen2014}, and Chinese data from the Simplified GSD Treebank, both included in the
Universal Dependencies corpus \cite{Nivre2020UniversalDV}. These datasets are already tokenized; for the Chinese data we used the existing tokenization, limited by a vocabulary size of $2^{14} - 2$ with added ``unknown'' and ``end-of-sentence'' tokens. For Finnish we also used the SentencePiece byte-pair encoding with a vocabulary size of $2^{14}$.

\begin{figure*}[t!]
    \centering
    \includegraphics[width=0.98\textwidth,clip,trim=0 0 0 1cm]{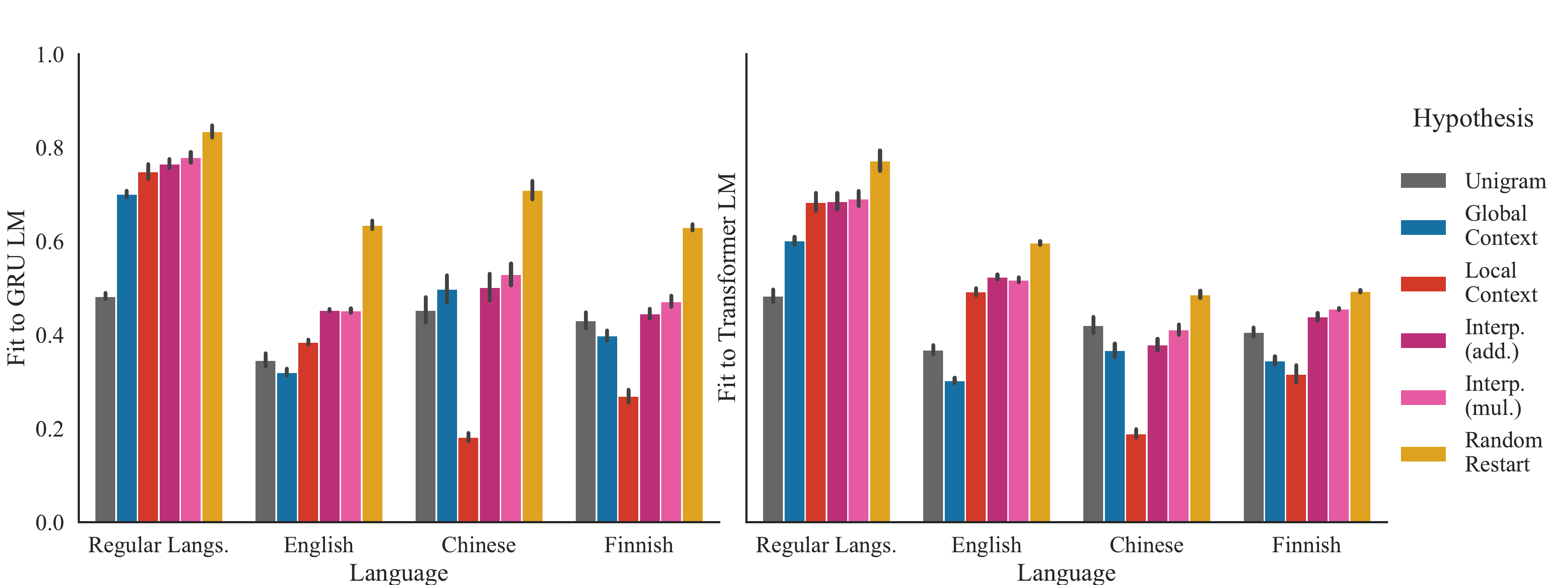}
    \vspace{-.5em}
    \caption{Accuracy $\textrm{acc}(\tilde{p}, \lm)$ of predicted generalization for various hypotheses $\tilde{p}$, with black lines showing one standard deviation across 5 (GRU) or 4 (transformer) random restarts. In some cases, generalization hypotheses are nearly as predictive as new neural models trained on the same data, suggesting that they explain most of extrapolation behavior that can be derived from data alone. Multiplicative interpolation is consistently a bit better than additive interpolation. Which of the two base hypotheses performs best (global or local generalization) varies substantially across languages. %
    }
    \label{fig:base}
\end{figure*}

To generate surprising natural language sentences $(\Xg, \Xl)$, we first select $\Xg$ by truncating sentences from the validation set to uniformly random lengths. We then run our best trained model $\lm$ to determine $\lm(X_{n-1} | \Xg)$, and choose a token $\Xl$ uniformly from among the set $\{X\ :\lm(X | \Xg) < \frac{1}{198}, X \in L\}$ where $L$ is the set of the 198 most-common tokens by unigram count (200 less the ``unknown'' and ``end-of-sentence'' tokens used in Chinese). In the framework of Equations \ref{eq:eps}--\ref{eq:tau}, $\epsilon$ is set to $\frac{1}{198}$ and $\tau$ to the smallest probability assigned in context to an in-distribution token.

To compute generalization model predictions on natural language data, we estimate $\plocal$ from bigram counts in the training set: $\plocal(X_n | X_L) = \frac{\text{count}(X_n, X_L)}{\text{count}(X_L)}$. To estimate $\pglobal$, we train a second, random restart of the model $\lmp$. We then estimate $\pglobal$ using one step of beam search in $\lmp$ with a beam of size 15:
\begin{multline}
\label{eq:beam}
\pglobal(X_n | \Xg) \\ \approx \sum_{i = 1}^{15} \lmp(v_i | \Xg) \lmp(X_n | \Xg, v_i)
\end{multline}
where the $v_i$ range over the 15 top predicted tokens after $\Xg$. Given a trained model that performs well on the \textit{in-distribution} validation set, we will have $\lm(v | \Xg) \approx p(v | \Xg)$ and $\lm(X_n | X_G, v) \approx p(X_n | X_G, v)$, and therefore that \cref{eq:beam} gives a good approximation of $\pglobal$.\footnote{We compute this quantity with a second language model in order to prevent information about the true model's out-of-distribution behavior from leaking into our prediction.}

For each natural language datasets, we trained GRUs with 2 hidden layers and word embedding and hidden sizes of 1024, and transformers with 4 heads, 2 layers, and hidden sizes of 512. All models were optimized with Adam using a learning rate of 3e-4 on shuffled length-aligned batches of up to 128 for 15 epochs. The model with the best held-out performance was then selected.

\subsection{Which model of generalization fits best?}
\label{sec:experiments-base}

Given a dataset of surprising contexts $\{(\Xg, \Xl)_i \}$, a hypothesis $\hyp$, and a trained model $\lm$, we compute the accuracy of the hypothesis $\hyp$ as
\begin{align}
\label{eq:acc}
    &\text{acc}(\hyp, \lm) = 1 - \text{err}(\hyp, \lm)
\end{align}
where
\begin{align}
    &\text{err}(\hyp, \lm) = \nonumber \\
    &\qquad \frac{1}{n}\sum_{\Xg, \Xl} \delta(\lm(\cdot \mid \Xg, \Xl), \hyp(\cdot \mid \Xg, \Xl)) \nonumber
\end{align}
and $\delta$ is the total variation distance
\begin{align}
    \delta(p_1, p_2) = \frac{1}{2} \| p_1 - p_2 \|_1 ~ .
\end{align}
In other words, we measure the accuracy of each hypothesis by computing the average $\ell_1$ distance between the hypothesized and true probability histograms across surprising contexts.
$\text{acc}(\hyp, \lm)$ is between 0 and 1; a large value indicates that $\hyp$ is a good approximation to $\lm$.\footnote{The use of a bounded distance measure is important---intuitively, we should regard a hypothesis as accurate even if it makes very inaccurate predictions in a small fraction of contexts. However, the specific choice of measure does not appear to be very important; defining $\textrm{err}$ in terms of the Jensen--Shannon divergence gives similar results. }

For hypotheses $\hyp$, we use (1) the local and global context (\cref{sec:models-base}), and (2) \emph{optimal} linear and log-linear interpolations between them (\cref{sec:models-interp}), choosing settings for $\lambda$ that minimize error on the evaluation set itself. To provide context for these results, we report the accuracy of a \textbf{unigram baseline}, which predicts the unigram distribution $p(X_n)$ independent of the context. We additionally report the error obtained by a \textbf{random restart}---a new model $p_{\theta'}$ trained from scratch (with a different initialization) on the same data as $\lm$; this model provides a rough upper bound on how much of $\lm$'s prediction can be explained by structural properties of the data distribution itself. $\mathrm{err}$ was computed from 100 samples for regular languages and 200 samples for natural languages.

Results are shown in \cref{fig:base}. 
In each language, either the local or global model is a good fit to observed generalization behavior. There are substantial differences across languages: the local context model is a good predictor for regular languages and English, but a poor predictor for Finnish and Chinese, suggesting that generalization behavior is data-dependent but not tied to word order or morphological complexity. In general, GRU generalization is more predictable than transformer generalization.
Finally, interpolation often substantially outperforms either base hypothesis (\cref{fig:interp-curve}), with log-linear interpolation slightly more predictive than linear interpolation. In several cases, interpolated hypotheses approach the accuracy of a randomly retrained predictor, suggesting that they capture much of the generalization behavior that is determined by the underlying data alone.

\begin{figure}
    \centering
    \vspace{-.5em}
    \includegraphics[width=0.9\columnwidth,clip,trim=0.2in 0in 0.1in 0.1in]{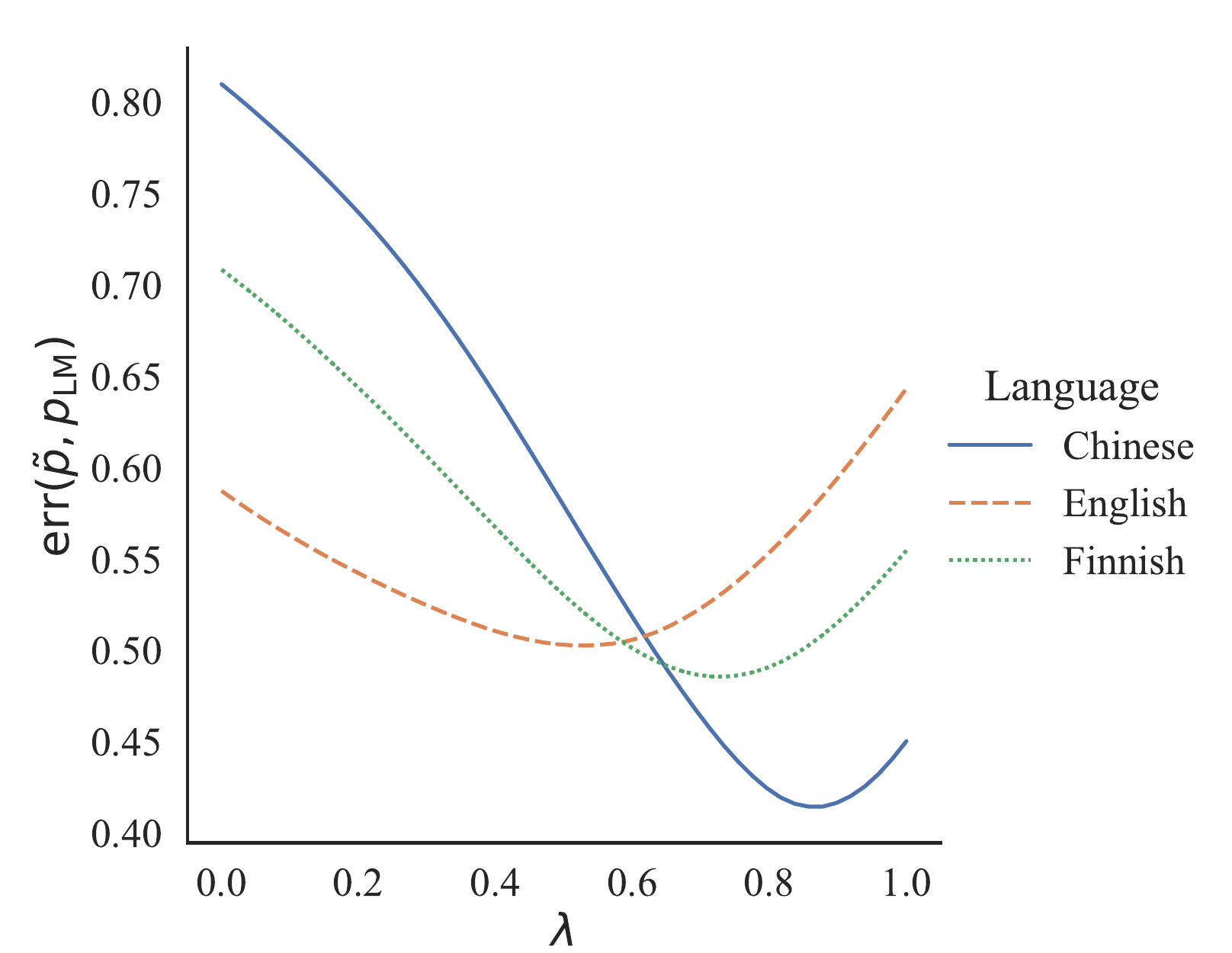}
    \vspace{-1em}
    \caption{Effect of the log-linear interpolation parameter $\lambda_1$ (fixing $\lambda_2 = 1 - \lambda_1$; \cref{eq:llinterp}) when predicting out-of-distribution behavior in  language models. As shown in \cref{fig:base}, in English, an all-local hypothesis ($\lambda=0$) is better than an all-global hypothesis ($\lambda=1$), but true model behavior is best approximated by a log-linear combination of the two ($\lambda \approx 0.5$). Finnish and Chinese are also best approximated by an interpolation, but are closer to the global than the local hypothesis. \vspace{-1em}}. 
    \label{fig:interp-curve}
\end{figure}

\subsection{What controls interpolation?}
\label{sec:experiments-interp}

The previous section showed that $\pinterpmul$ gives the best fit to the empirical distribution of neural LM predictions across contexts: out-of-distribution prediction in both RNNs and transformers involves a mix of global and local information, with the precise weighting of these two sources of information dependent on structural properties of the language being modeled. A natural next question is whether this weighting can be \emph{controlled}: that is, whether modifications can be made to models or training procedures that affect the relative importance of the local and global hypotheses.

In this section, we explore \emph{noise} as a possible source of this control. Models of both perceptual (local) noise and retrieval (global / contextual) noise play a key role in computational models of human sentence processing \cite{levy2008noisy}. In machine learning, various kinds of noise injected at training time---most prominently dropout \cite{srivastava2014dropout}, but also label noise and random word substitution and masking---are widely used as tools to regularize model training and limit overfitting. Here, we investigate whether these noising procedures qualitatively affect the \emph{kind} of generalization behavior that neural LMs exhibit in the out-of-distribution contexts explored in \cref{sec:experiments-base}.

We investigate two kinds of noise: random word substitution and hidden state dropout. In all experiments, this noise is applied at training time only; model inference is run noiselessly when evaluating fit in \cref{eq:acc}.
When computing $\pglobal$ with \cref{eq:beam} in these experiments, $\lmp$ is also trained without noise to approximate $\pglobal$.

\paragraph{Random token substitution} With probability $p$, input tokens are randomly replaced with samples from the unigram distribution. Random word substitution plays an important role in masking-based pretraining schemes \citep{devlin2018bert}.

\paragraph{Hidden state dropout} With probability $p$, features of context representations (RNN hidden states and transformer self-attention outputs) are randomly set to zero \citep{semeniuta2016recurrent}.

\paragraph{}

Results are shown in \cref{fig:interp}. Across languages, word substitution modestly increases the predictive accuracy of the global context model, while sometimes decreasing the accuracy of the local context model.
For natural languages, however, the improvement in the global context model is matched by improvements in the \textit{unigram} baseline, indicating that these results may simply indicate decreased context-dependence. In regular languages, symbol-swapping noise increases the performance of the global hypothesis without improving the baseline, suggesting reliance on global information has actually increased. In all cases, state dropout improves the predictive accuracy of the local context model and often decreases the accuracy of both the global context model and the unigram baseline, suggesting indicating that state dropout encourages local generalization.

\begin{figure}[t]
    \centering
    \includegraphics[width=\columnwidth,clip,trim=0.25in 2in 2.5in 0.25in]{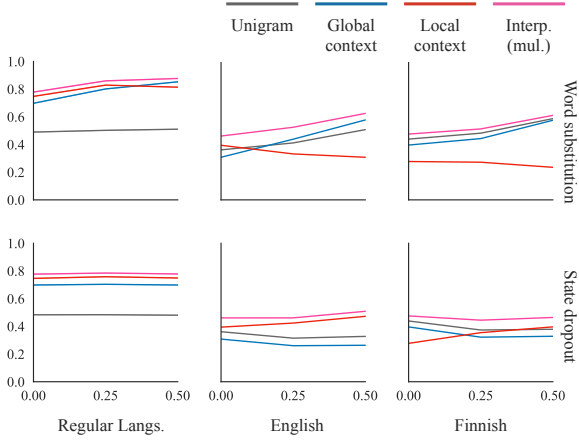}
    \caption{Accuracy $\textrm{acc}(\tilde{p}, \lm)$ of predicted generalization for English, Finnish, and regular-language GRUs when trained with token-swapping noise and state dropout noise. Token swapping sometimes improves the accuracy of the global context model, while state noising improves the accuracy of the local context model. Similar trends occur with Chinese; see Appendix \ref{appendix:extra-interp}.}
    \label{fig:interp}
\end{figure}

\section{Explaining the experiments}
\label{sec:theory}

With empirical evidence that interpolation between local and global models is a good approximation to out-of-distribution language model behavior, we next investigate whether this behavior can be explained theoretically. While we leave for future work a complete answer to this question, 
we conclude with the following proposition, which describes a set of conditions under which LM generalization will be well approximated by log-linear interpolation between $\plocal$ and $\pglobal$.

\begin{proposition}
   \label{prop:main}
   Let $\theta$ be the parameters of a log-linear model optimizing:
   \begin{multline}
        \argmin_\theta \\ - \sum_{X_{1:n}} \log p(X_{n} \mid X_{1:n-1}; \theta) + \lambda \|\theta\|^2
   \end{multline}
   where $p(X_n \mid X_{1:n-1}) \propto \exp\{ \theta_{X_n}^\top \phi(X_{1:n-1}) \}$ and $\phi(\cdot)$
   has an indicator feature for each value of $\Xg$, $\Xl$, and the conjunction $(\Xg, \Xl)$. Suppose further that models with \emph{only} local or global features are boundedly worse than this model: specifically, that:
   \begin{equation}
       \expected |p(X_n \mid \Xg, \Xl) - p(X_n \mid \tilde{X})| < \epsilon
   \end{equation} 
   uniformly
   for \emph{training} $\tilde{X}$ equal to either $\Xg$ or $\Xl$. Then, in surprising contexts, $p(x_n \mid x_{1:n-1}; \theta)$ can be approximated by $\tilde{p}_\times$:
   \begin{align}
        &\Big|p(X_n \mid \Xg, \Xl; \theta) \nonumber \\ & ~~~~ - ~ \tilde{p}_\times(X_n \mid \Xg, \Xl) \Big| < e^{4 \epsilon / \lambda} - 1
   \end{align}
    where
    $\tilde{p}_\times(X_n \mid \Xg, \Xl) \propto \tilde{p}(X_n \mid \Xg)~\tilde{p}(X_n \mid \Xl)$
    and each $\tilde{p}(X_n \mid \Xl \text{ or } \Xg)$ is an $\ell_2$-regularized estimate of the corresponding distribution.
\end{proposition}
\noindent
In other words, for log-linear language models with informative local and global features, the observed effectiveness of multiplicative interpolation is expected. 
Proof is given in \cref{appendix:proofs}. 

It is important to qualify this result in several ways: it relies on a feature function $\phi$ that may not be a realistic representation of the context feature produced by deep network models, involves strong assumptions about the independent predictive power of local and global features at training time, and becomes vacuous for large values of $\epsilon$ or small values of $\lambda$. 
It is weak in absolute sense ($e^{4\epsilon/\lambda}$ grows considerably faster than $\epsilon$ except for very large values of $\lambda$); its function is simply to relate predictions in surprising contexts to measurable properties of the training distribution.
Nevertheless, the result shows that some aspects of interpolation behavior can be predicted from the parametric form of predictors alone; future work might strengthen this claim to more directly characterize the neural network predictors studied in this paper and explain the observed differences across languages.

\section{Conclusion}

When neural sequence models are exposed to out-of-distribution contexts with conflicting local and global information, their behavior can be predicted. Across natural and synthetic data distribution, sequence model generalization appears to be well approximated by either a local ($n$-gram-like) or global (skip-gram-like) predictor, and best approximated by a log-linear interpolation of the two whose weight can sometimes be controlled by noise-based regularization.
This work suggests several avenues for future exploration: first, explaining \emph{data-dependent} aspects of the local--global tradeoff (especially cross-linguistic differences that are not clearly explained by typological differences between languages); second, determining whether \emph{architectural} improvements to standard sequence models can even more effectively target specific kinds of structured generalization.

\section*{Acknowledgments}

This research was supported by the MIT--IBM Watson AI Lab.

\bibliographystyle{acl_natbib}
\bibliography{anthology,custom}

\onecolumn

\appendix

\section{Noise Experiments in Other Settings}
\label{appendix:extra-interp}

We include larger versions of the graphs from Figure \ref{fig:interp} with a few other hypotheses added: the additive interpolation and random restart along with an \emph{ignore} hypothesis that predicts LM generalization having consistent with having never seen the surprising token (i.e.\ using the predictive distribution from \emph{before} the surprising token was observed).

\begin{figure}[H]
    \centering
    \textbf{GRU models on regular languages}\par\medskip
    \includegraphics[width=0.45\textwidth]{"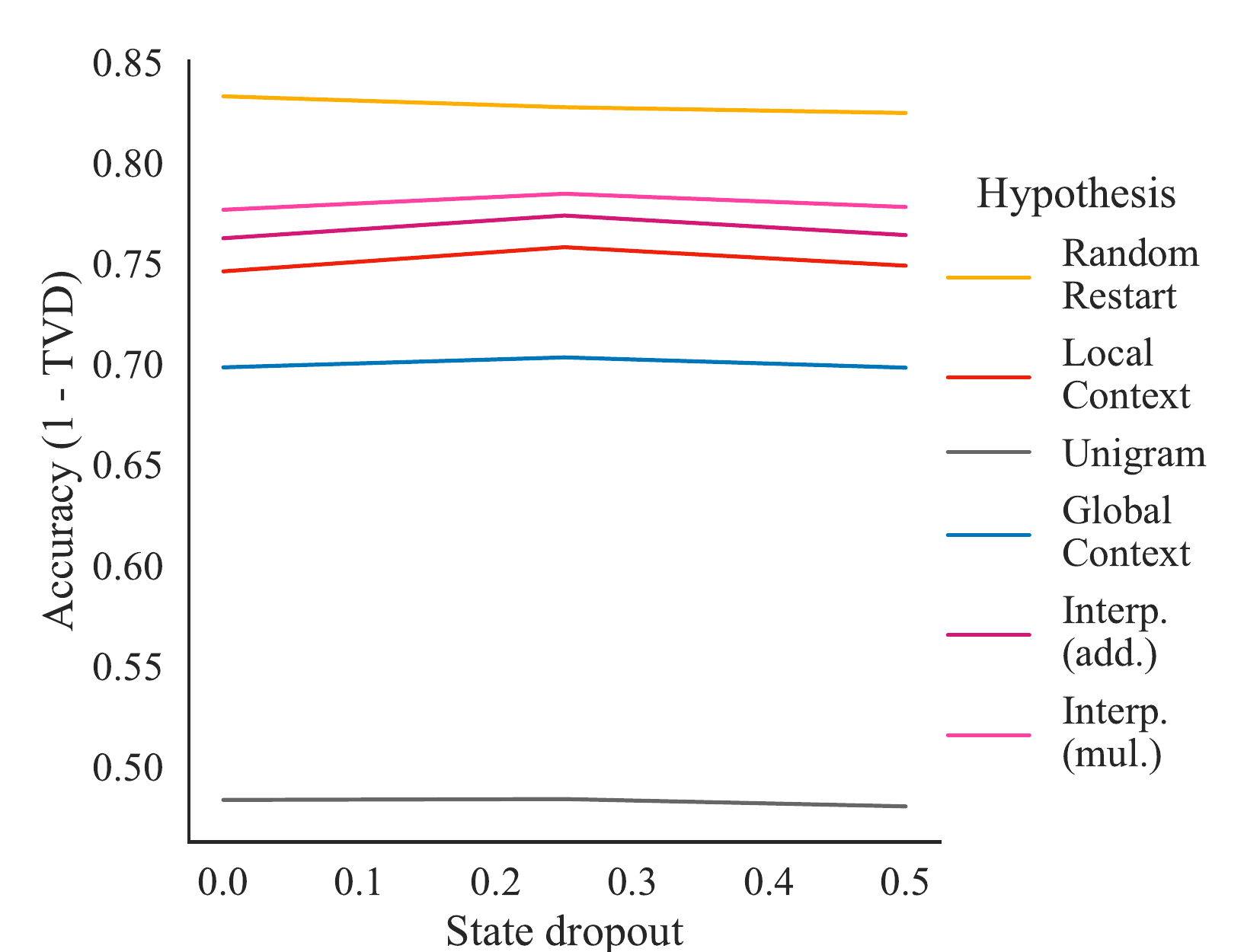"}
    \includegraphics[width=0.45\textwidth]{"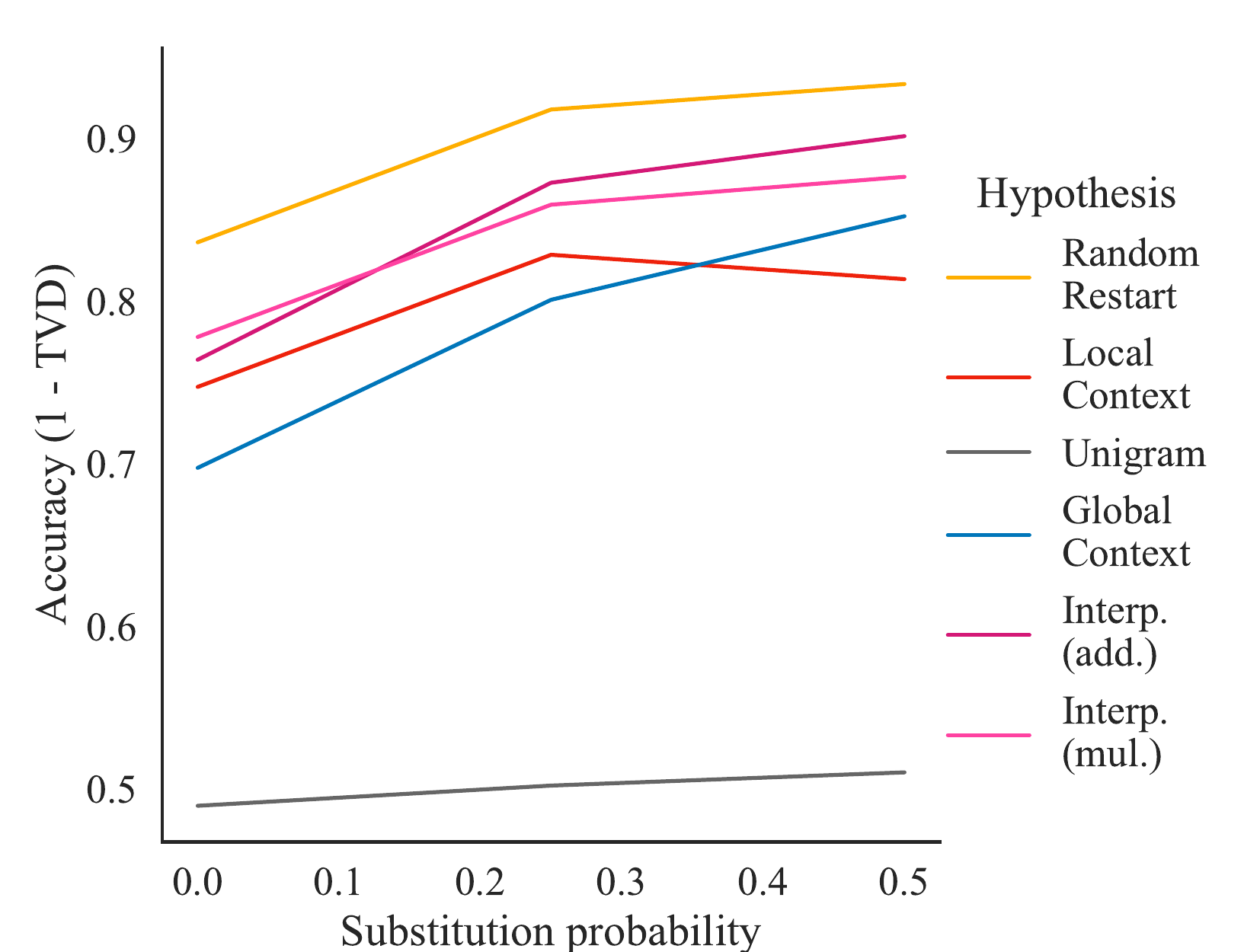"}
    \caption{Detail version of graphs from Figure \ref{fig:interp}.}
    \label{fig:regular-interp}
\end{figure}
\begin{figure}[H]
    \centering
    \textbf{GRU models on English}\par\medskip
    \includegraphics[width=0.45\textwidth]{"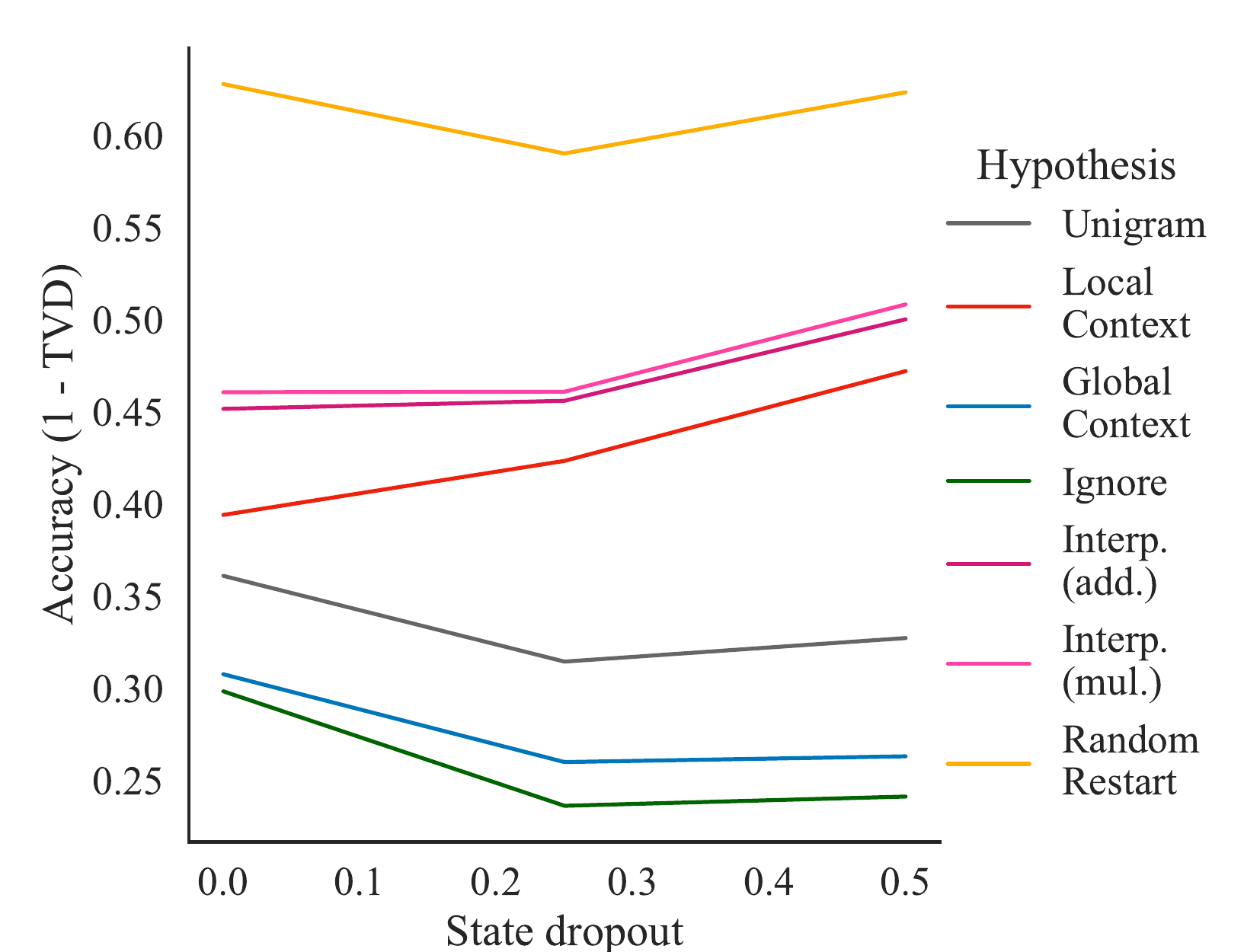"}
    \includegraphics[width=0.45\textwidth]{"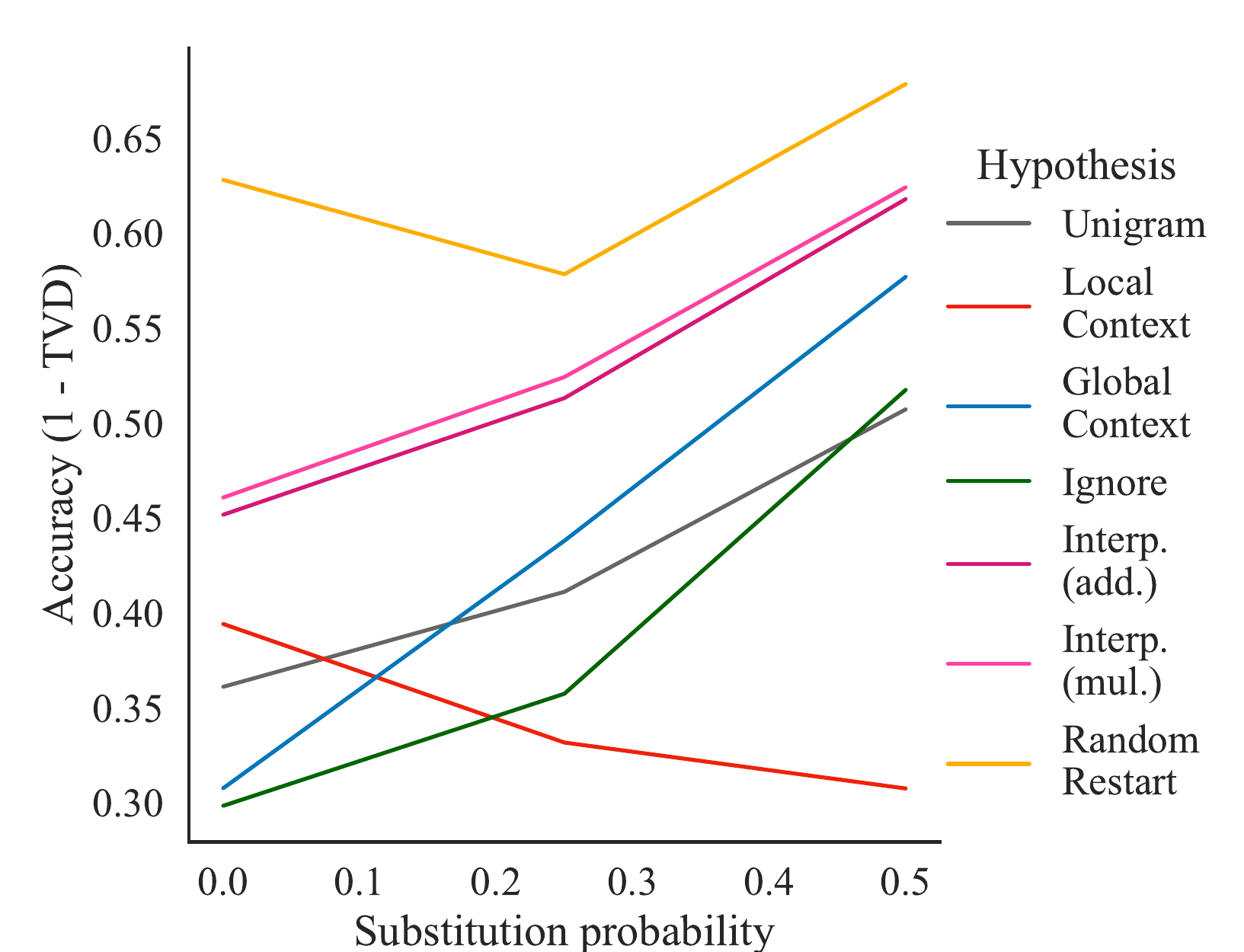"}
    \caption{Detail version of graphs from Figure \ref{fig:interp}.}
    \label{fig:english-interp}
\end{figure}
\begin{figure}[H]
    \centering
    \textbf{GRU models on Finnish}\par\medskip
    \includegraphics[width=0.45\textwidth]{"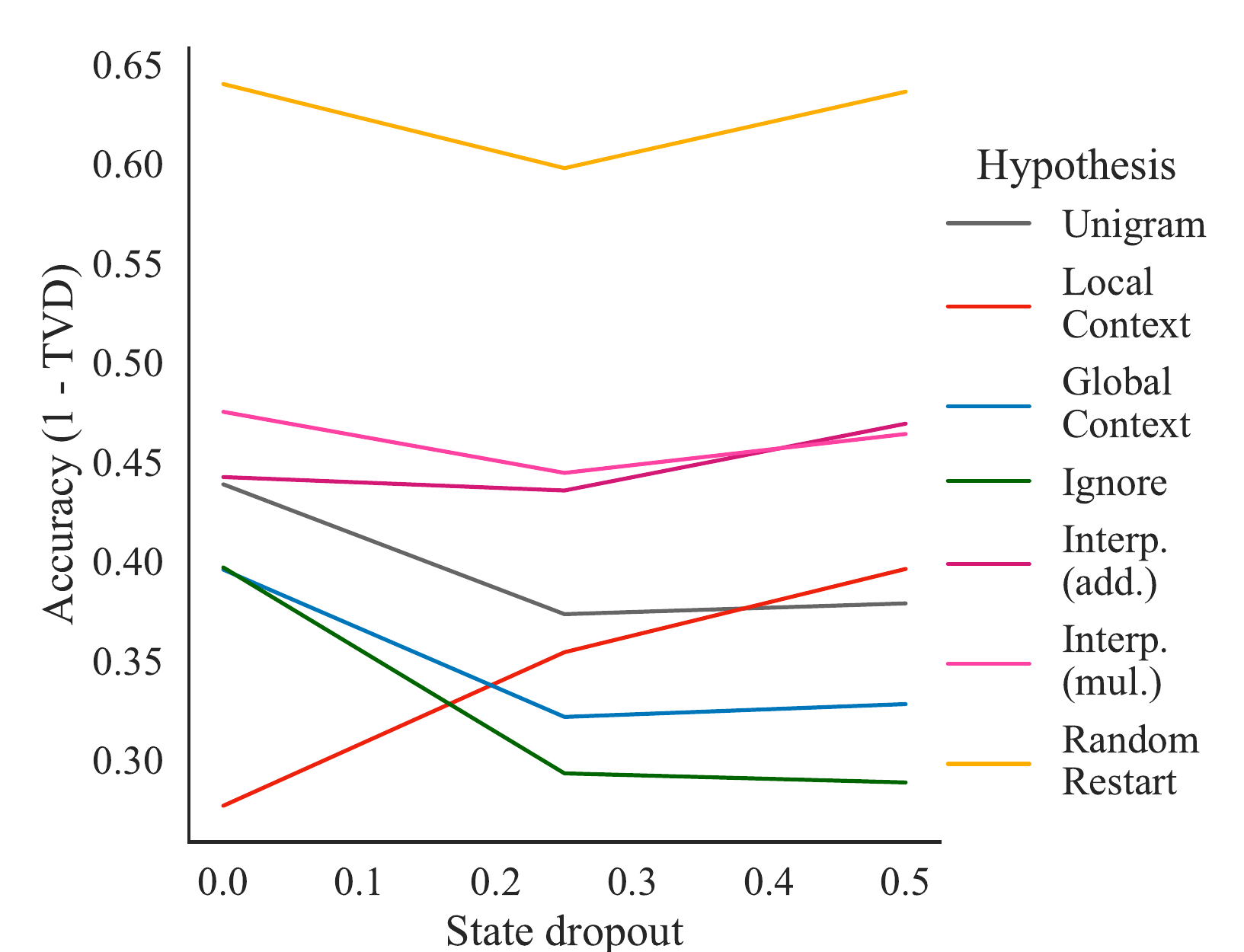"}
    \includegraphics[width=0.45\textwidth]{"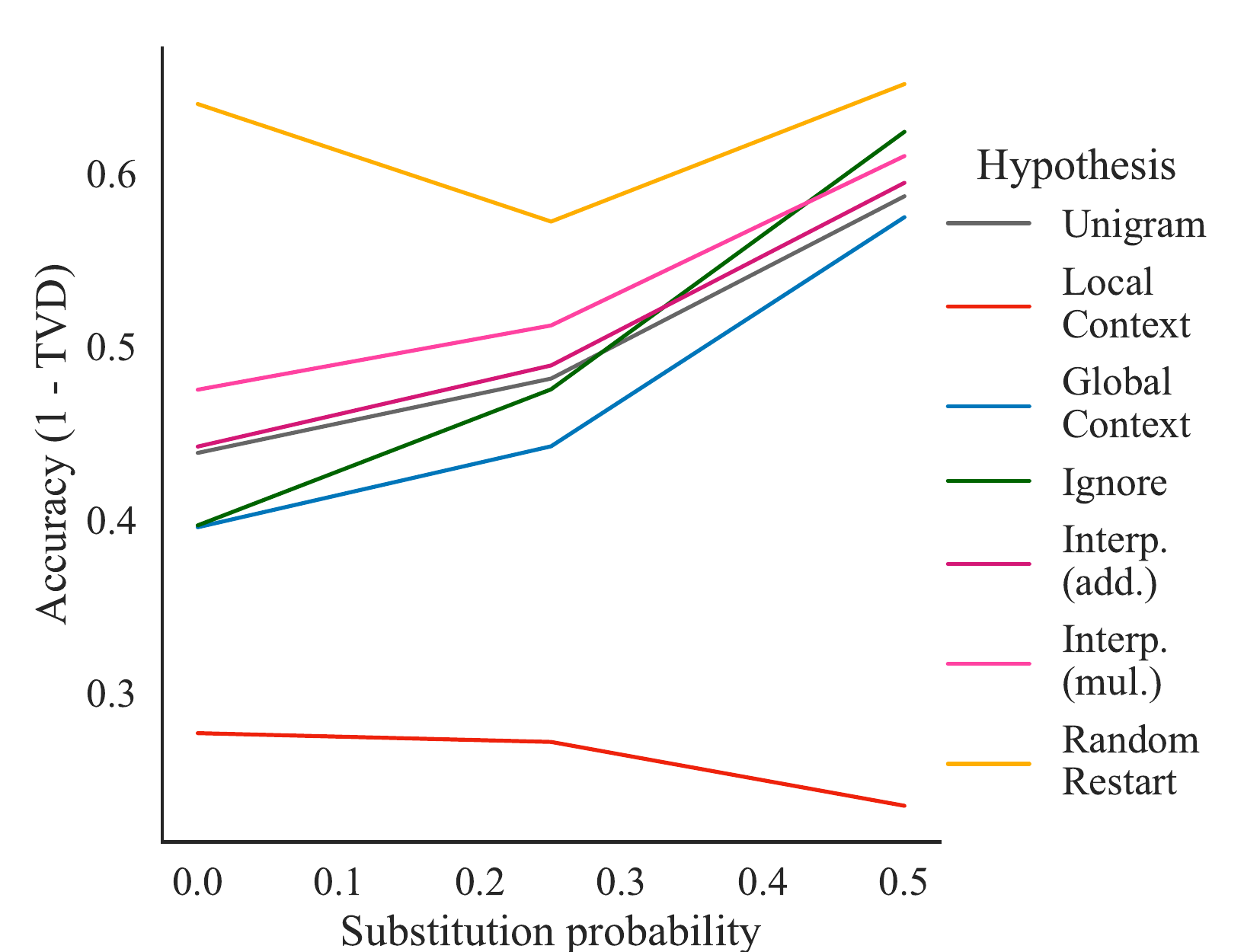"}
    \caption{Detail version of graphs from Figure \ref{fig:interp}.}
    \label{fig:finnish-interp}
\end{figure}

We include graphs of some settings for the noise experiments of Figure \ref{fig:interp} that were omitted from the main paper for space reasons: Chinese GRU models and transformer models on regular languages. The trends are mostly similar.

\begin{figure}[H]
    \centering
    \textbf{GRU models on Chinese}\par\medskip
    \includegraphics[width=0.45\textwidth]{"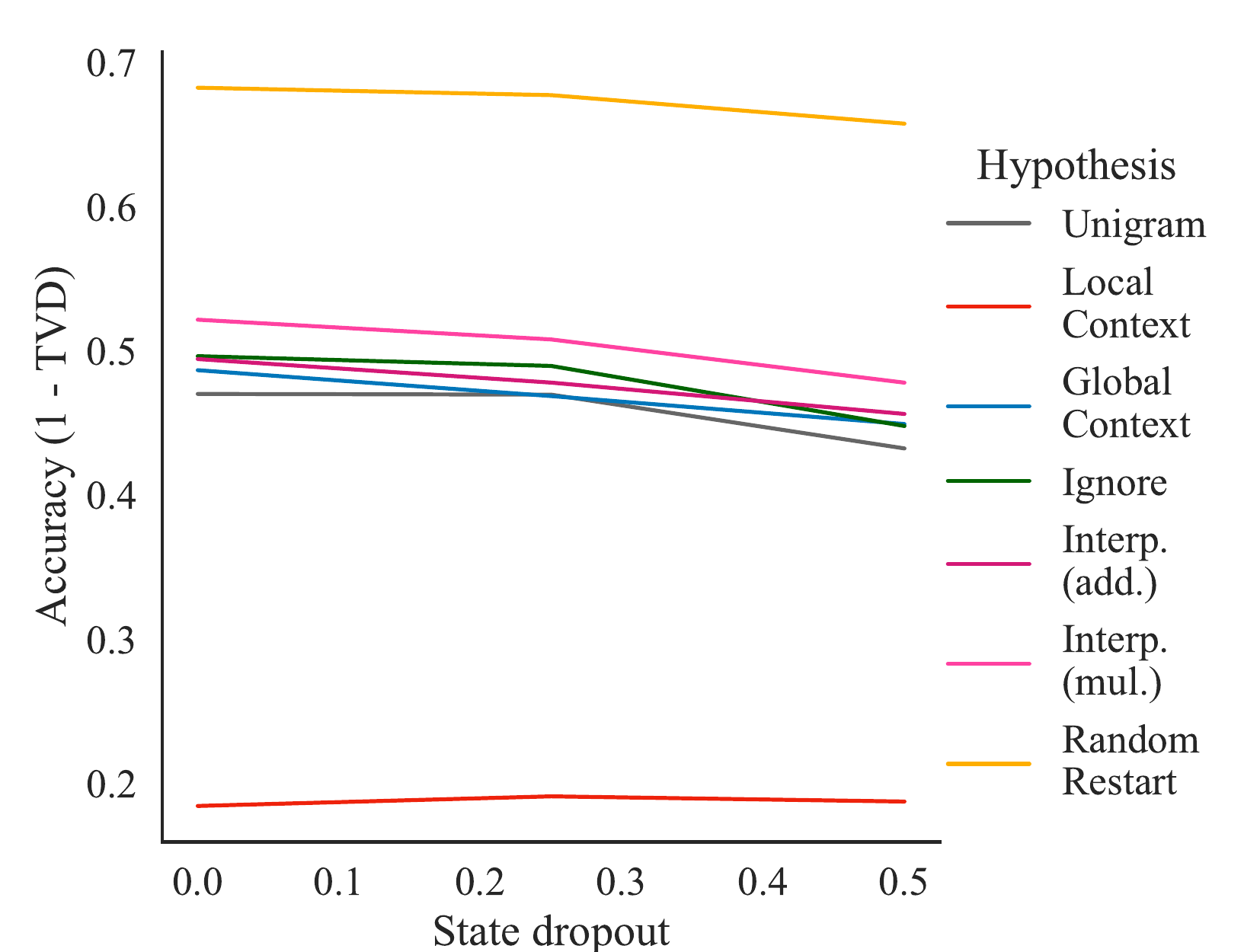"}
    \includegraphics[width=0.45\textwidth]{"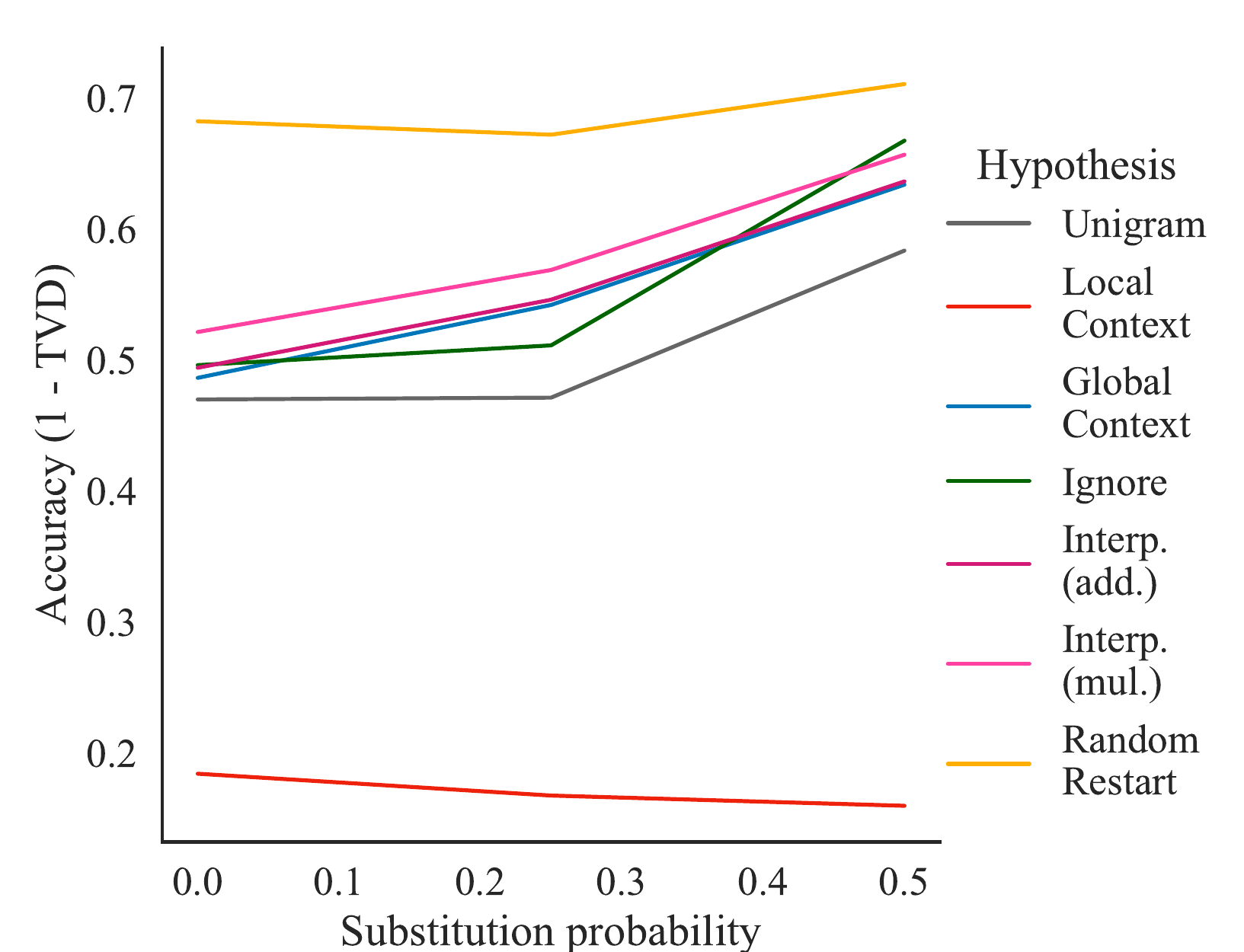"}
    \caption{Graphs showing hypothesis performance for Chinese GRU language models trained under the same noise conditions as in Figure \ref{fig:interp} and similar trends.}
    \label{fig:chinese-interp}
\end{figure}

\begin{figure}[H]
    \centering
    \textbf{Transformer models on regular languages}\par\medskip
    \includegraphics[width=0.45\textwidth]{"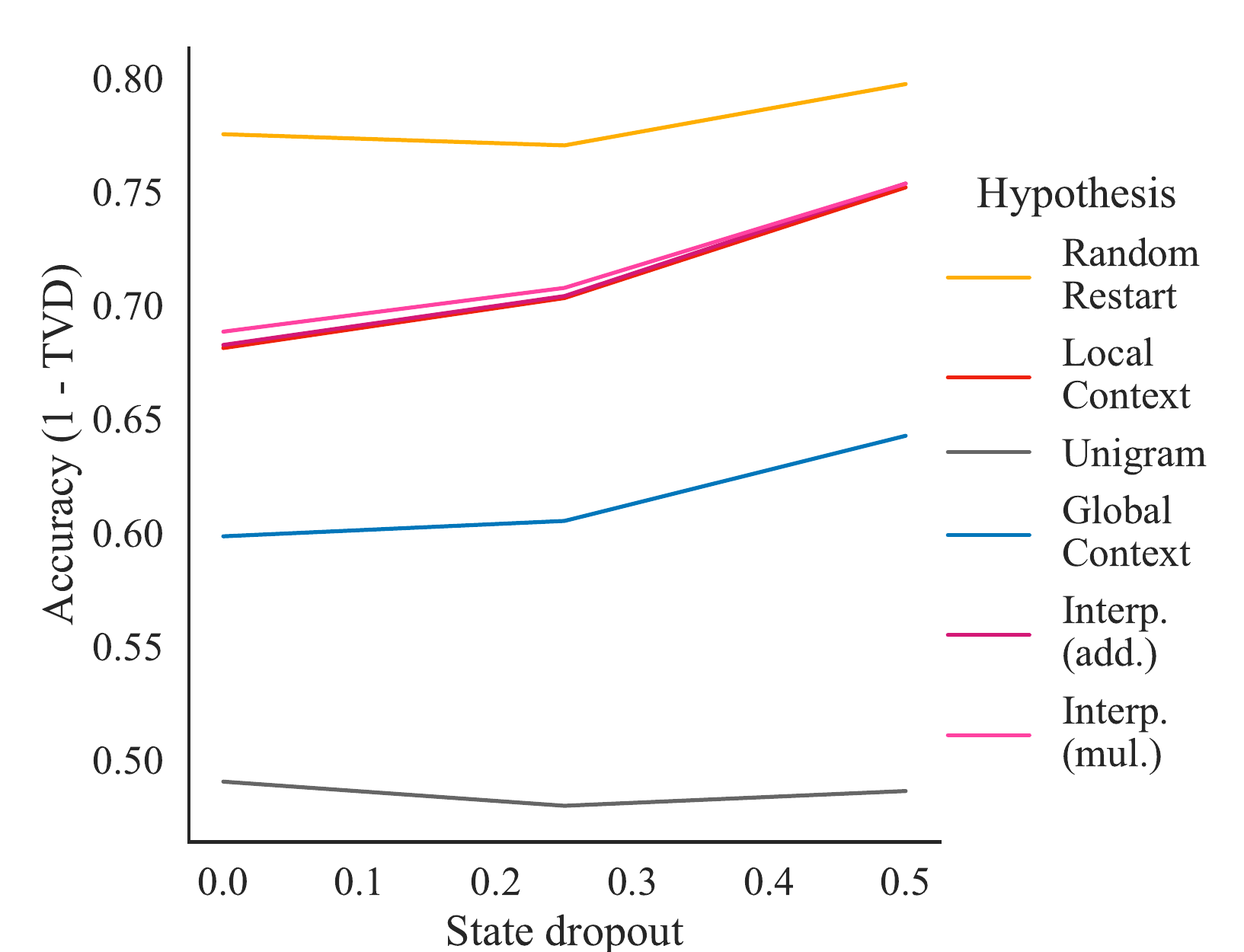"}
    \includegraphics[width=0.45\textwidth]{"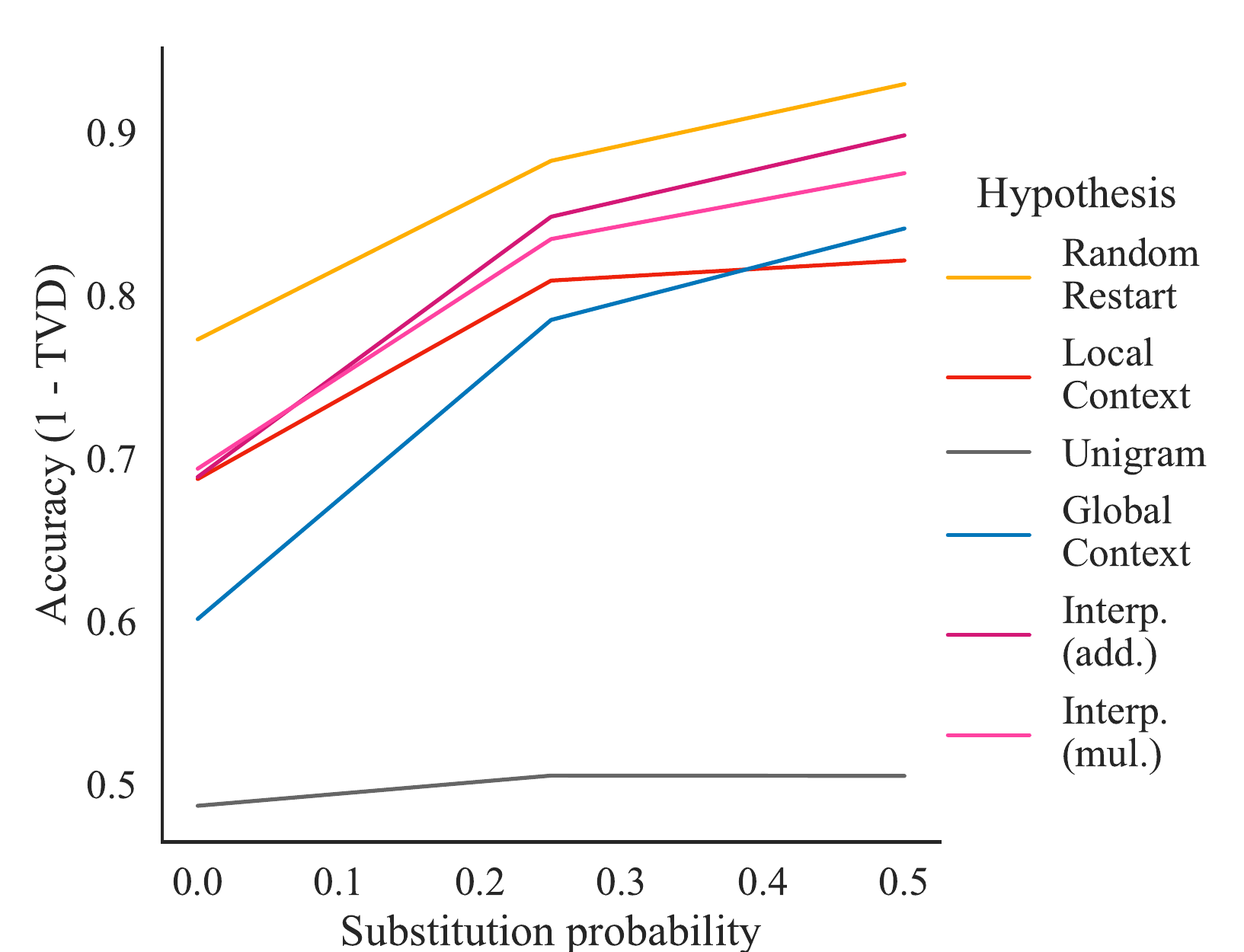"}
    \caption{Graphs showing hypothesis performance for transformer language models for regular languages, trained under the same noise conditions as in Figure \ref{fig:interp} and similar trends.}
    \label{fig:transformer-interp}
\end{figure}

\section{Constructing Finite Automata}
\label{appendix:dfas}

\lstset{language=Python, basicstyle=\ttfamily\footnotesize}
\begin{lstlisting}
input:
    symbols: Set[char]
    states: Set[int]
    num_symbol_uses: int
    num_neighbors: int
    start_state: int
    
while true:
    symbol_counts = dict(s -> num_symbol_uses for s in symbols)
    dfa_edges = set()
    dfa_accepting_states = set()
    for state in states:
        neighbors = uniform(states).sample_without_replacement(n=num_neighbors)
        used_symbols = set()
        for neighbor in neighbors:
            while True: 
                symbols_available = set(
                    s for s in symbols 
                    if symbol_counts[s] > 0 and s not in used_symbols
                )
                if symbols_available.is_empty():
                    break
                symbol = uniform(available_symbols).sample()
                dfa_edges.add(edge(state, symbol, neighbor))
                used_symbols.add(symbol)
                symbol_counts[symbol] -= 1
                if used_symbols.size == symbols.size / states.size:
                    break
        if bernoulli(0.5).sample():
            dfa_accepting_states.add(state)
            
    dfa_edges = prune_unreachable_edges(dfa_edges, start_state)
    if dfa_edges.is_empty():
        continue
        
    return dfa
\end{lstlisting}

\section{Proof of \cref{prop:main}}
\label{appendix:proofs}

We begin with a simple lemma relating parameter weights in regularized log-linear models with mismatched feature sets.

\begin{lemma}
\label{lemma:main}
Let $\phi^1$ and $\phi^2$ be feature distinct functions producing binary feature vectors, with $\phi^1_i(x) = \phi^2_i(x)$ for some $i$. Let $y_1, y_2, \ldots$ be output classes, and $\theta^1 = [\theta^1_{y_1}, \theta^1_{y_2}, \ldots]$ and $\theta^2$ be the result of optimizing:
\begin{align}
    \label{eq:objective}
    &\argmin_\theta -\sum_{x, y} \log p(y \mid x; \theta) + \lambda\|\theta\|^2 \\
    &=\argmin_\theta - \sum_{x, y} \theta_y^\top \phi(x) - \log \sum_{y'} \exp \{ \theta_{y'}^\top \phi(x) \} + \lambda \|\theta\|^2
\end{align}
for each $\phi = \phi^1, \phi^2$. As in \cref{prop:main}, suppose the two feature functions produce similar predictors in the sense that: 
\begin{equation}
       \expected |p(y \mid x, \theta^1) - p(y \mid x, \theta^2)| < \epsilon
\end{equation}
for all training $x$.
Then,
\begin{equation}
    |\theta^1_{y,i} - \theta^2_{y,i}| \leq \frac{\epsilon}{\lambda}
\end{equation}
\end{lemma}

\begin{proof}
\cref{eq:objective} is convex, and at optimality its gradient with respect to each component of $\theta$ is 0. Then, for any $\theta_{v, i}$ we have:
\begin{align}
    0 &= \sum_{x, y = v} \phi_i(x) (-1 + p(y \mid x; \theta))
    + \sum_{x, y \neq v} \phi_i(x) p(y \mid x; \theta) + \lambda \theta_{v,i} \\
    \theta_{y,i} &= \frac{1}{\lambda} \Big( \sum_{x, y = v} \phi_i(x) (1 - p(y \mid x; \theta))
    + \sum_{x, y \neq v} - \phi_i(x) p(y \mid x; \theta) \Big) \\
    |\theta^1_{y,i} - \theta^2_{y,i}| &= \frac{1}{\lambda} \sum_{x, y} \phi_i \big| p(y \mid x; \theta_1) - p(y \mid x; \theta_2) \big| \\
    &\leq \frac{\epsilon}{\lambda}
\end{align}
\end{proof}

\noindent
We can then obtain the main result:
\begin{proof}[Proof of \cref{prop:main}]

\newcommand\thetagg{\eta}
\newcommand\phig{\phi}
\newcommand\thetal{\mu}
\newcommand\phil{\phi}

First note that we can estimate both local and global models using distributions of the form:
\begin{equation}
    p(X_n \mid \Xg) \propto \exp \{ \thetagg^\top \phig(\Xg) \}
\end{equation}
where $\phig(\Xg)$ contains only global features.
(and similarly for $\Xl$). If these models are trained with the same regularization constant $\lambda$, the conditions of \cref{lemma:main} will be satisfied with respect to each local or global feature. Then,
\begin{align}
    &\Big|p(X_n \mid X_{1:n-1}; \theta) ~ - ~ (1/Z) ~ p(X_n \mid \Xg) ~ p(X_n \mid \Xl)\Big|
    \\
    &= 
    \Bigg| \frac{\exp \{ \theta_{X_n}^\top \phi(\Xg, \Xl) \}}{\sum_v \{\exp \theta_v^\top \phi(\Xg, \Xl) \}} ~ - ~ \frac{\exp \{ \thetagg_{X_n}^\top \phig(\Xg) \} \cdot \exp \{ \thetal_{X_n}^\top \phil(\Xl) \} }{\sum_v \exp \{ \thetagg_v^\top \phig(\Xg) \} \cdot \exp \{ \thetal_{X_n}^\top \phil(\Xl) \} } \Bigg| \\
    \intertext{Here we have used $\theta$ to denote the parameters of the full model, $\thetagg$ to denote the parameters of the global-only model, and $\thetal$ to denote the parameters of the local-only model. Because each $\phi$ has an indicator for local, global, and (local, global) values, only two features are active in surprising contexts, corresponding to $\Xg$ and $\Xl$. We will denote the indices of these weights $i$ and $j$ in each weight vector:}
    &= \Bigg| \frac{\exp \{ \theta_{X_n, i} + \theta_{X_n, j} \} }{\sum_v \exp \{ \theta_{v, i} + \theta_{v, j} \} } - \frac{\exp \{ \thetagg_{X_n, i} + \thetal_{X_n, j} \} }{\sum_v \exp \{ \thetagg_{v,i} + \thetal_{v, j} \} } \Bigg|\\
    \intertext{Without loss of generality, assume the first term is larger than the second. By applying \cref{lemma:main}, we can rewrite $\theta$ in terms of $\thetagg$ and $\thetal$. For shorthand, we will write $\Delta = \epsilon / \lambda$:}
    &\leq \frac{\exp \{ \thetagg_{X_n, i} + \Delta + \thetal_{X_n, j} + \Delta \} }{\sum_v \exp \{ \thetagg_{v, i} - \Delta + \thetal_{v, j} - \Delta \} } - \frac{\exp \{ \thetagg_{X_n, i} + \thetal_{X_n, j}\} }{\sum_v \exp \{ \thetagg_{v,i} + \thetal_{v, j} \} } \\
    &= \Big(e^{4\Delta}-1\Big)  \frac{\exp \{ \thetagg_{X_n, i} + \thetal_{X_n, j} \} }{\sum_v \exp \{ \thetagg_{v,i} + \thetal_{v, j} \} } \\
    &\leq e^{4 \epsilon / \lambda} - 1
\end{align}
Note that this proof is not specific to local and global feature representations, but applies generically to \emph{any} pair of regularized log-linear models with overlapping features and similar predictions. The proof is adapted from Theorem 1 of \citet{zou2005regularization}, and we expect that it could be strengthened (as done there) to depend only on correlations between features of $\phi(\Xg, \Xl)$ and each of $\{ \phi(\Xg), \phi(\Xl) \}$.
\end{proof}

\end{document}